\documentclass[20pt]{article}  %conference IEEEtran
\usepackage{hyperref}
\usepackage{bbm}
\usepackage{amsfonts}
\usepackage{graphicx,subfig}
\usepackage{amsmath,amsthm}
\usepackage{wrapfig}

\usepackage{epsfig}
\usepackage{amsmath}
\usepackage{amssymb}
\usepackage{psfrag}
\usepackage[absolute]{textpos}

\newtheorem{thm}{Theorem}
\newtheorem{lem}{Lemma}
\newtheorem{define}{Definition}
\newtheorem{cor}{Corollary}

\def\x{{\bf x}}

\def\v{{\bf v}}
\def\m{{\bf m}}
\def\n{{\bf n}}
\def\u{{\bf u}}
\def\e{{\bf e}}
\newcommand{\aaa}{{\bf a}}
\newcommand{\bbb}{{\bf b}}
\newcommand{\1}{\mathbbm{1}^{n\times n}}
\newcommand{\one}{\mathbbm{1}}

\newcommand{\ri}{\right>}
\newcommand{\li}{\left<}

\newcommand{\su}{{\text{sum}}}

\newcommand{\la}{\lambda}
\newcommand{\A}{{\bf{A}}}

\newcommand{\M}{\mathbf{M}}
\newcommand{\MM}{\mathcal{M}}
\newcommand{\N}{\mathbf{N}}

\newcommand{\PPP}{{\bf{P}}}
\newcommand{\U}{{\bf{U}}}

\newcommand{\V}{{\bf{V}}}

\newcommand{\R}{\mathbb{R}}
\newcommand{\Rr}{\mathcal{R}}
\newcommand{\Aa}{\mathcal{A}}

\newcommand{\C}{\mathcal{C}}
\newcommand{\G}{\mathcal{G}}
\newcommand{\HH}{\mathcal{H}}
\newcommand{\E}{\mathbb{E}}
\newcommand{\X} {{\bf{X}}}
\newcommand{\s}{\star}

\newcommand{\eps}{\epsilon}

\newcommand{\z}{{\bf z}}

\newcommand{\beq}{\begin{equation}}
\newcommand{\eeq}{\end{equation}}
\newcommand{\bear}{\begin{align}}
\newcommand{\eear}{\end{align}}
\newcommand{\bea}{\begin{eqnarray}}
\newcommand{\eea}{\end{eqnarray}}

\newcommand{\Prob}{\ensuremath{\mathbb{P}}}
\long\def\symbolfootnote[#1]#2{\begingroup%
\def\thefootnote{\fnsymbol{footnote}}\footnote[#1]{#2}\endgroup}

\begin{document}

\title{Finding Dense Clusters via ``Low Rank $+$ Sparse'' Decomposition}

\author{Samet Oymak, Babak Hassibi\vspace{10pt}\\
 California Institute of Technology\vspace{2pt} \\ \{soymak,hassibi\}@caltech.edu
\thanks{This work was supported in part by the National Science Foundation under grants CCF-0729203, CNS-0932428 and CCF-1018927, by the Office of Naval Research under the MURI grant N00014-08-1-0747, and by Caltech's Lee Center for Advanced Networking.}
}
\maketitle

\begin{abstract}
Finding ``densely connected clusters'' in a graph is in general an important and well studied problem in the literature \cite{Schaeffer}. It has various applications in pattern recognition, social networking and data mining \cite{Duda,Mishra}. Recently, Ames and Vavasis have suggested a novel method for finding cliques in a graph by using convex optimization over the adjacency matrix of the graph \cite{Ames, Ames2}. Also, there has been recent advances in decomposing a given matrix into its ``low rank'' and ``sparse'' components \cite{Candes, Chandra}. In this paper, inspired by these results, we view ``densely connected clusters'' as imperfect cliques, where imperfections correspond missing edges, which are relatively sparse. We analyze the problem in a probabilistic setting and aim to detect disjointly planted clusters. Our main result basically suggests that, one can find \emph{dense} clusters in a graph, as long as the clusters are sufficiently large. We conclude by discussing possible extensions and future research directions.
\end{abstract}

%Introduction
%
%
%
%
%
%
\section{Introduction}
Recently, convex optimization methods have become increasingly popular for data analysis. For example, in compressed sensing \cite{Candes2}, we observe the measurements and aim to recover an unknown sparse solution of a system of linear equations via $\ell_1$ minimization. In many other cases, we have the perfect knowledge of a signal which possibly looks complicated, however it has a simpler underlying structure and we aim to reveal this structure by decomposing it into meaningful pieces. For example, decomposing a signal into a sparse superposition of sines and spikes is one of the well-known problems of this type \cite{Candes3}. Decomposing a matrix into low rank and sparse components is another key problem of this nature and it has recently been studied in various settings \cite{Candes,Chandra,Sengha,Ganesh}. In this problem, we observe the matrix $L^0+S^0$, where $L^0$ is low rank and $S^0$ is sparse, and we aim to find $L^0$ and $S^0$. The suggested convex optimization program is as follows:
\begin{align}
\label{optim1}
&\min \|L\|_\s+\lambda \|S\|_1\\
\nonumber&\text{subject to}\\
\nonumber&\hspace{20pt}L+S=L^0+S^0
\end{align}
Here $\|\cdot\|_\s$ is the nuclear norm i.e. sum of the singular values of a matrix and $\|\cdot\|_1$ is the $\ell_1$ norm, i.e., the sum of the absolute values of the entries. Problem (\ref{optim1}) can be considered as the natural convex relaxation of ``low rank $+$ sparse'' decomposition as $\ell_1$ norm and nuclear norm are the tightest convex relaxations of the sparsity and rank functions respectively. Consequently, this program promotes sparsity for $S$ and low rankness for $L$. For the correct choice of $\lambda$, if $L^0$ and $S^0$ satisfies certain incoherence requirements, it is known that we'll have $(L^*,S^*)=(L^0,S^0)$ where $(L^*,S^*)$ is output of problem (\ref{optim1}).

This result is actually very useful as low rankness and sparsity are the underlying structures in many problems. In \cite{Venkat}, Gaussian graphical models with latent variables were investigated and the problem of finding conditional dependencies of the observed variables was connected to problem (\ref{optim1}). On the other hand, in the problem of finding cliques in a given unweighted graph, the key observation is the fact that, in the adjacency matrix, a clique corresponds to a submatrix of all $1$'s which is clearly rank $1$. Based on these observations, in this paper, we aim to extend the results of Ames and Vavasis \cite{Ames,Ames2} for detection of the planted cliques to detection of ``densely connected clusters''. This problem comes up naturally, as most of the times, it might be unreasonable to expect full-cliques in a graph. For example, there might be missing edges naturally, or data might be corrupted or we might be observing only partial information. However, even if we miss some of the edges, it is very likely that most of the edges will be preserved and the cluster, we want to identify, will still be \emph{denser} than the rest. We'll view these dense clusters as imperfect cliques with some missing edges and, in our approach, full cliques will correspond to low rank piece, $L^0$, whereas the missing edges inside and (extra) edges outside of the clusters will correspond to sparse piece, $S^0$.

We analyze the problem under a general probabilistic setting, which we call ``probabilistic cluster model''. In our model, an edge inside $i$'th cluster exists with probability $p_i$ and an edge which is not inside any of the clusters exists with probability $q$ independent of other edges in the graph, where $1\geq p_i>q\geq 0$ are constant. Here, by ``inside a cluster'', we mean an edge lying between two nodes which belong to the same cluster. Notice that, this model can be viewed as a slight modification of well-known Erd\"os-R\'enyi random graph model where we introduce a nonuniform distribution which makes the clusters identifiable. We additionally assume the clusters are disjoint.

We'll analyze two convex programs for detection of the clusters using the knowledge of the graph. We name the first program ``blind approach'' and it is just a slight modification to problem (\ref{optim1}), given in (\ref{optim3}), and we show that if 
\beq
\min_ip_i=p_{min}>\frac{1}{2}>q
\eeq as long as the clusters are sufficiently large, with high probability, problem (\ref{optim3}) can detect the clusters. Our second program is called ``intelligent approach'' which is given in problem (\ref{optim2}). In this case, we require an extra information but we can guarantee the detection for any $p_{min}>q$. Problem (\ref{optim2}) can be considered as a mixture of (\ref{optim1}) and (12) of \cite{Ames} because it focuses on the subgraph induced by the edges inside the clusters similar to (12) of \cite{Ames} but additionally accounts for the missing edges. This approach also trivially extends to the case where we observe the partial graph, in which, each edge is observed with same probability independent of others. In this case, clusters can still be recovered but we need clusters to be slightly larger compared to the case we observe the full graph.

%Basic Definitions and Notations
%
%
%
%
%
%
%
\section{Basic Definitions and Notations}
Let $[m]$ denote the set $\{1,2,\dots,m\}$ for all integers $m\geq 1$. We differentiate a subset of nodes in a graph by calling that subset a cluster. For the rest of the paper, we assume the graph $\G$ is unweighted with $n$ nodes, and there are $t$ {{\bf{disjoint}}} planted clusters with sizes $\{k_i\}_{i=1}^t$ nodes. By unweighted we mean edges do not carry weights. Assume nodes are labeled from $1$ to $n$ and let $\C_i$ be the set of the nodes inside the cluster hence $\C_i\subseteq[n]$, $|\C_i|=k_i$ and $\C_i\cap\C_j=\emptyset$ for any $i\neq j$. We also let $\C_{t+1}$ denote rest of the nodes i.e. $\C_{t+1}=[n]-\bigcup_{i=1}^t \C_i$ and $k_{t+1}=n-\sum_{i=1}^tk_i$.

We call a subset $\beta$ of $[c]\times [d]$, a \emph{region}. $\beta^c$ denotes the complement, which is given by $\beta^c=[c]\times [d]-\beta$.

Let $\Rr$ be the region corresponding to the union of regions induced by the clusters, i.e., $\Rr=\bigcup_{i=1}^t \C_i\times \C_i$. Note that $\Rr$ is simply a subset of $[n]\times[n]$. We also let $\Rr_{i,j}=\C_i\times \C_j$ for $1\leq i,j\leq t+1$. $\{\Rr_{i,j}\}$ basically divides $[n]\times [n]$ into $(t+1)^2$ disjoint regions similar to a grid. Also $\Rr_{i,i}$ is simply the region induced by $i$'th cluster for any $i\leq t$.

Let $a,b\in\R$ and $0\leq r\leq 1$. We say a random variable $X$ is $\text{Bern}(a,b,r)$ if
\begin{align}
&\Prob(X=a)=r\\
&\Prob(X=b)=1-r
\end{align}

For a given matrix $\X$, $\X_{i,j}=(\X)_{i,j}$ denotes the entry lying on $i$'th row and $j$'th column. $\one^{c\times d}$ is a $c\times d$ matrix where entries are all $1$'s. Assume $\beta$ is a subset of $[c]\times [d]$. Then, $\beta$ can be viewed as a set of coordinates and if $\X\in\R^{c\times d}$, we denote the matrix which is induced by entries of $\X$ on $\beta$ by $\X_\beta$:
\beq
(\X_\beta)_{i,j}=\begin{cases}\X_{i,j}~~\text{if}~(i,j)\in \beta\\
0~~\text{else}
\end{cases}
\eeq
In particular, $\one^{c\times d}_\beta$ is a matrix, whose entries on $\beta$ are $1$ and rest of the entries are $0$. 

Now, we introduce some definitions to explain the model we'll work on.

\begin{define} [Random Support] \label{support}A random set $\beta\subseteq [c]\times [d]$ is called ``random support'' with parameter $0\leq r\leq 1$ if each coordinate $(i,j)\in[c]\times[d]$ is an element of $\beta$ with probability $r$, independent of other coordinates.

A random set $\Gamma\in[c]\times[d]$ is called ``corrected random support'' with parameter $r$ if it is statistically identical to $\beta\cup \bigcup_{i=1}^{\min\{c,d\}} (i,i)$ where $\beta$ is a random support with parameter $r$. Basically, we include the diagonal coordinates.
\end{define}

Let $\A$ be the adjacency matrix of $\G$. For simplicity, we let $\A_{i,i}=1$ for all $i\in[n]$. Also for $(i,j)\in[n]\times[n],~i\neq j$
\beq
\A_{i,j}=\begin{cases}1~~\text{if an edge exists between nodes}~i,j\\0~~\text{else}\end{cases}
\eeq
Note that, $\A$ is symmetric, i.e., $\A_{i,j}=\A_{j,i}$ for all $i,j\in[n]$, as a result, it is uniquely determined by the entries on the lower triangular part.

\begin{define}[Probabilistic Cluster Model]
\label{model}
Recall that $\C_i\subseteq [n]$ with $\C_i\cap \C_j=\emptyset$ and $|\C_i|=k_i$ for all $1\leq i\neq j\leq t$. Also $\Rr_{i,j}=\C_i\times \C_j$ for all $1\leq i,j\leq t+1$ and $\Rr=\bigcup_{i=1}^t \Rr_{i,i}$. Let $\{p_i\}_{i=1}^t,q$ be constants between $0$ and $1$. Then, a random graph $\G$, generated according to probabilistic cluster model, has the following adjacency matrix. Entries of $\A$ on the lower triangular part are independent random variables and for any $i> j$:
\begin{align}
\label{prob}
\A_{i,j}=\begin{cases}\text{Bern}(1,0,p_l)~\text{random variable if}~(i,j)\in\Rr_{l,l}~\text{for some}~l\leq t\\
\text{Bern}(1,0,q)~\text{random variable}~\text{else}
\end{cases}
\end{align}
\end{define}

Verbally, an edge inside $l$'th cluster exists with probability $p_l$ and an edge which is not inside any of the clusters exists with probability $q$, independent of other edges, where $1\geq \{p_l\}_{l=1}^t,q\geq 0$. In order to distinguish clusters we'll assume they are denser i.e. an edge inside the region $\Rr$ is more likely to exist compared to an edge which is not. Consequently, we have:
\beq
\label{pmineq}
\min_{i\leq t}p_i=p_{min}>q
\eeq
for the rest of the paper. One can similarly treat the case where $\max_{i\leq t} p_i=p_{max}<q$ by considering the complement graph $\HH$ whose adjacency $B_{ij}=1-A_{ij}$ for all $i\neq j$. In this case, $\HH$ will still satisfy probabilistic model with inside and outside cluster edge probability $\{1-p_i\}_{i=1}^t,1-q$ respectively where $\min_{i\leq t} 1-p_i>1-q$. Notice that, in the special case of cliques, we have $p_i=1$ for all $i\leq t$.

In this model, $\A$ can be characterized also by using random supports.
\beq
\A=\sum_{i=1}^t\1_{\Rr_{i,i}\cap \beta_i}+\1_{\Rr^c\cap\Gamma}
\eeq
where $\{\beta_i\},\Gamma$ are independent corrected random supports with parameters $\{p_i\},q$ respectively.

Let $\Aa\subseteq[n]\times[n]$ be the set of nonzero coordinates of $\A$, i.e., $\1_\Aa=\A$. Basically, $\Aa$ is the region induced by the edges inside the graph $\G$ with the addition of diagonal coordinates. For example, the set $\Aa^c\cap \Rr$ corresponds to the missing edges inside the clusters. Clearly, $\Aa$ is random, as $\G$ is drawn from probabilistic cluster model.

We'll call a matrix (or vector) positive (negative) if all its entries are positive (negative). Finally, we let $\su(\X)$ denote sum of the entries of $\X$ i.e. $\su(\X)=\sum_{i=1}^c\sum_{j=1}^d \X_{i,j}$ for $\X\in\R^{c\times d}$. If matrix $X$ is nonnegative then $\su(X)=\|X\|_1$.

%Proposed Programs
%
%
%
%
%
%
%
\section{Proposed Convex Programs}
Our aim is finding the clusters $\{\C_i\}_{i=1}^t$ in a graph $\G$ drawn from the probabilistic cluster model described in Definition \ref{model}. This can be achieved by finding $\Rr$. This is not hard to see, because, in the matrix $\1_\Rr$, nonzero entries of each column will exactly correspond to one of the clusters, as clusters are disjoint. Then, we can simply scan through all columns to find the clusters.

\subsection{Blind Approach}
As our first approach, in order to find $\Rr$, we suggest the following, slightly modified version of problem (\ref{optim1})
\begin{align}
\label{optim3}
&\min_{L,S} \|L\|_\s+\lambda\|S\|_1\\
\nonumber&\text{subject to}\\
%\nonumber&\hspace{15pt} S_{i,j}\geq 0~~\text{for all}~i,j\\
\label{line12}
&\hspace{15pt} 1\geq L_{i,j}\geq 0~~\text{for all}~i,j\\
\label{line22}
&\hspace{15pt} L+S=\A%\text{trace}((\1-\A)^T(L-S))=0\\
%\nonumber&\hspace{15pt} \su(L)\geq|\Rr|=\sum_{i=1}^t k_i^2
\end{align}
%This convex program has a strong similarity to problem (12) of \cite{Ames} which aims to recover a single planted clique. The fundamental differences are that we additionally account for missing edges in a similar manner to problem \ref{optim1} and we are dealing with recovery of more than one disjoint clusters. As we have mentioned in the introduction, we incorporate the knowledge of the size of the region of clusters i.e. $|\Rr|=\sum_{i=1}^t k_i^2$. Actually, with this knowledge, we'll be able  guess the solution of problem (\ref{optim2}) (under the right assumptions).

Advantage of this approach is the fact that we don't need any additional information about clusters such as number (or sizes) of the clusters. The desired solution is $(L^0,S^0)$ where $L^0$ corresponds to the full cliques, when missing edges inside $\Rr$ are completed, and $S^0$ corresponds to the missing edges and the extra edges between the clusters. In particular we want:
\begin{align}
\label{optimal}
&L^0=\1_\Rr\\
&S^0=\1_{\Aa\cap \Rr^c}-\1_{\Aa^c\cap\Rr}
\end{align}
It is easy to see that the $(L^0,S^0)$ pair is feasible, later we'll argue that under correct assumptions $(L^0,S^0)$ is indeed unique optimal solution.

\subsection{Intelligent Approach}
The second convex problem to be analyzed is a mixture of problems (\ref{optim1}) and (12) of \cite{Ames}. We'll require an extra information which is the size of the region induced by clusters, i.e., $|\Rr|$. Suggested program focuses on subgraph induced by the edges inside the clusters and is given below:
\begin{align}
\label{optim2}
&\min_{L,S} \|L\|_\s+\lambda\|S\|_1\\
\nonumber&\text{subject to}\\
\label{line1}
&\hspace{15pt} 1\geq L_{i,j}\geq S_{i,j}\geq 0~~\text{for all}~i,j\\
\label{line2}
&\hspace{15pt} \text{trace}((\1-\A)^T(L-S))=0\\
\label{line3}
&\hspace{15pt} \su(L)\geq|\Rr|=\sum_{i=1}^t k_i^2
\end{align}
Actually, knowledge of $|\Rr|$,will help us guess the solution of problem (\ref{optim2}) (under the right assumptions). $L^0$ should correspond to the full cliques similar to (\ref{optim2}), however $S^0$ should only correspond to the missing edges inside the clusters. Formally, we want:
\begin{align}
\label{optimal2}
&L^0=\1_\Rr\\
&S^0=\1_{\Aa^c\cap\Rr}
\end{align}
In the next section, we'll state the main results of the paper regarding the problems \ref{optim3} and \ref{optim2}. Proofs of the theorems in section \ref{mainresult} will be given in sections \ref{intelproof}, \ref{blindproof} and \ref{secconv}. Finally, section \ref{concsec} will conclude the paper.

%Main Result
%
%
%
%
%
\section{Main Results}
\label{mainresult}
In this section, we'll explain the conditions for which the candidates given in (\ref{optimal}) and (\ref{optimal2}) are the unique optimal solutions of problems (\ref{optim3}) and (\ref{optim2}) respectively. This will also naturally answer the question of finding the densely connected clusters $\{\C_i\}_{i=1}^t$. Let $k_{min}$ be the size of the minimum cluster
\beq
k_{min}=\min_{1\leq i\leq t} k_i
\eeq
and $p_{min}$ was given previously in (\ref{pmineq}). Our analysis yields the following following fundamental constraints.
\begin{itemize}
\item $\la\sqrt{n}\leq C$ for some constant $C$ (In particular $\la=\frac{1}{2\sqrt{n}}$ will work).
\item $k_i> \frac{1}{\la(p_i-q)}$ for all $i\leq t$.
\end{itemize}
Actually, both of these constraints are natural. In \cite{Candes}, $\la=\frac{1}{\sqrt{n}}$ is used as the weight for problem (\ref{optim1}). It is not surprising that we are using a similar weight as our random graph model has strong similarities with the uniformly random support of the sparse component in \cite{Candes}. Secondly, we observe that $\la\leq \frac{C}{\sqrt{n}}$ implies $k_i>\frac{\sqrt{n}}{C(p_i-q)}$ which suggests that for recoverability, we need size of the $i$'th cluster to be at least $\Omega(\sqrt{n})$ and as $p_i-q$ gets smaller, this size should grow. This condition is consistent with the previous results of \cite{Ames,Ames2} which says for recoverability of $t$ disjoint cliques, one needs a minimum clique size of $\Omega(\sqrt{n})$.

The main results of this paper are summarized in the following theorems.
\begin{thm} [Main Result for Intelligent Approach] \label{mainthm} Set $\la=\frac{1}{2\sqrt{n}}$. Let $\G$ be a random graph generated according to the probabilistic cluster model (\ref{model}) with cluster sizes $\{k_i\}_{i=1}^t$ and parameters $\{p_i\},q$. Assume $p_{min}>q$ and $k_i\geq \frac{2}{\la(p_i-q)}=\frac{4\sqrt{n}}{p_i-q}$ for all $i\leq t$. Then, (independent of rest of the parameters) there exists constants $c,C>0$ such that as a result of convex program (\ref{optim2}) we have
\begin{align}
&L_0=\1_\Rr=\sum_{i=1}^t \1_{\Rr_{i,i}}\\
&S_0=\1_{\Aa^c\cap\Rr}
\end{align}
with probability at least (w.p.a.l.)
\beq
\label{concentrate1}
1-cn^2\exp(-C(p_{min}-q)^2k_{min})
\eeq
\end{thm}
In Theorem \ref{mainthm}, one can simplify the condition on $\{k_i\}$ by simply requiring $k_{min}\geq \frac{4\sqrt{n}}{p_{min}-q}$ however statement of the theorem will be weaker unless all $\{p_i\}$'s are equal.
Following corollary gives an idea about the case where we observe the partial graph.
\begin{cor}[Result for Partially Observed Graphs] Let $\G$ be a random graph as described in Theorem \ref{mainthm} and we observe each edge of $\G$ with probability $r$ independent of the other edges. Let $\A'$ be the adjacency matrix of the observed subgraph. Then, statement of Theorem \ref{mainthm} holds with variables $\A',\{p'_i\},q'$ instead of $\A,\{p_i\},q$ respectively where $q'=rq$ and $p'_i=rp_i$ for all $i\leq t$. Hence for recovery (w.h.p.), we require $k_i\geq \frac{4\sqrt{n}}{r(p_i-q)}$ for all $i\leq t$.
\end{cor}

\begin{thm} [Main Result for Blind Approach] \label{mainthmb}
Let $p_{min}>\frac{1}{2}>q$ and $\G$ be a random graph generated according to the probabilistic cluster model with cluster sizes $\{k_i\}_{i=1}^t$. Set $\la=\frac{1}{4\sqrt{n}}$ and assume $k_i\geq \frac{8\sqrt{n}}{2p_i-1}$ for all $i\leq t$. Then, there exists constants $c,C>0$ such that as the output of problem \ref{optim3} we have
\begin{align}
&L_0=\1_\Rr\\
&S^0=\1_{\Aa\cap \Rr^c}-\1_{\Aa^c\cap\Rr}
\end{align}
with probability at least
\beq
\label{concentrate2}
1-cn^2\exp(-C(\min\{2p_{min}-1,1-2q\})^2k_{min})
\eeq
\end{thm}

We should emphasize that slightly stronger results can be given for both theorems. For example, we can reduce the lower bound required for $\{k_i\}$ by a factor of four in both theorems at the expense of the error exponent $C$. In fact, one can get even better lower bounds for $\{k_i\}$  by choosing $\la$ as a function of $\{p_i\},q$ however we preferred to make $\la$ independent of $\{p_i\},q$.

The following theorem provides a converse result for blind method.
\begin{thm}
\label{conversethm}
Let $\G$ be a random graph generated according to the probabilistic cluster model with $\{p_i\},q$ and assume $p_{min}>q$, $\la=\frac{C}{\sqrt{n}}$ for some constant $C>0$. Then, if \beq
\frac{1}{2}\geq p_{min}~~~~~\text{or}~~~~~q>\frac{1}{2}~\text{and}~\Rr\neq [n]\times [n]
\eeq
as $n\rightarrow\infty$, $(L^0,S^0)$ given in (\ref{optimal}) is not a minimizer of problem (\ref{optim3}) with probability approaching $1$.
\end{thm}
{\bf{Remark:}} Note that if $\Rr=[n]\times[n]$ there is nothing to solve as all nodes are in the same cluster.

%Conclusions
%
%
%
%
%
%
%
\section{Future Extensions and Conclusion}
\label{concsec}

\subsection{Simulation Results}
We considered two relatively small cases. For the first case, we have $t=2$, $n=64$, $c_1=c_2=28$, $q=0.15$ and $p_1=p_2=p$ is variable. We plotted the empirical probability of success for both methods as a function of $p$ in Figure \ref{fig1}.

Secondly, in order to illustrate the difference between intelligent and blind approaches, we set $t=1$, $n=50$, $c_1=40$, $q=0.10$ and varied $p_1=p$. Due to Theorem \ref{conversethm}, for blind approach to work, we always need $p>1/2$. On the other hand, intelligent approach will work for any $p>q$ as long as $k_{min}$ is sufficiently large. Hence, when we increase $k_{min}$ we expect to see a better recovery region for intelligent approach compared to blind. We should remark that, in a probabilistic setting, $t=1$ case is trivial as we can find the cluster with high probability by looking at the nodes with high degree. Empirical curve is given in Figure \ref{fig2}. 

{\bf{Remark:}} In order to keep the model size $n$ small, we used $\la=\frac{1}{\sqrt{n}}$ in both of the simulations.

\begin{figure}[t]
\centering
  \includegraphics[width= 1\textwidth]{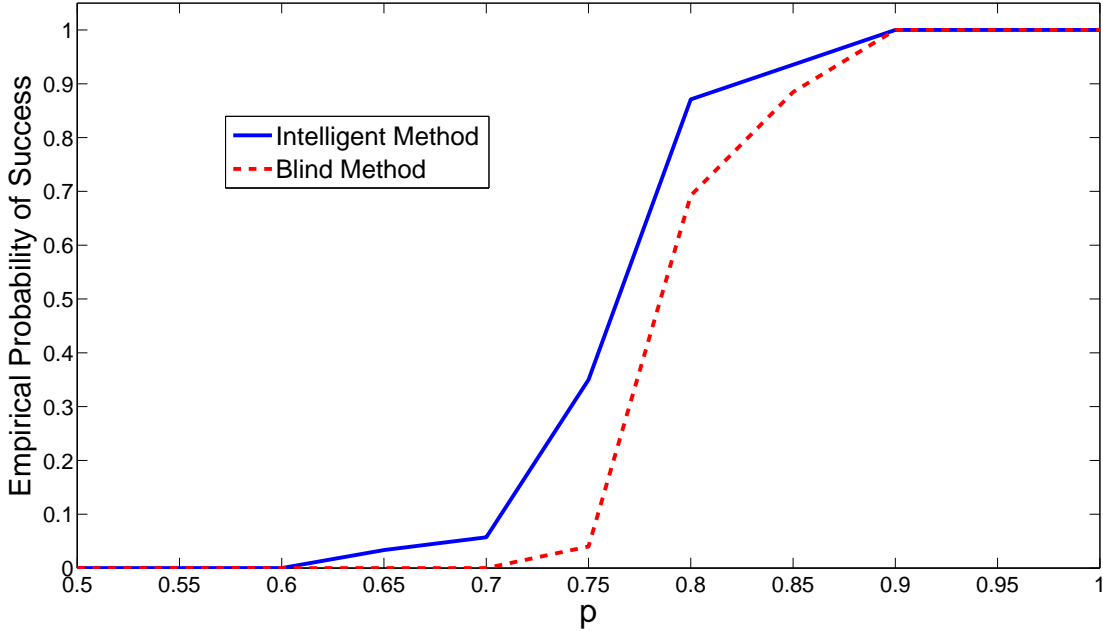}
 \caption{\scriptsize{Two methods perform close to each other. Also observe that phase transition is sharp and for $p>0.8$ both methods succeed w.h.p.}}
  \label{fig1}
\end{figure}
\begin{figure}[t]
\centering
  \includegraphics[width= 1\textwidth]{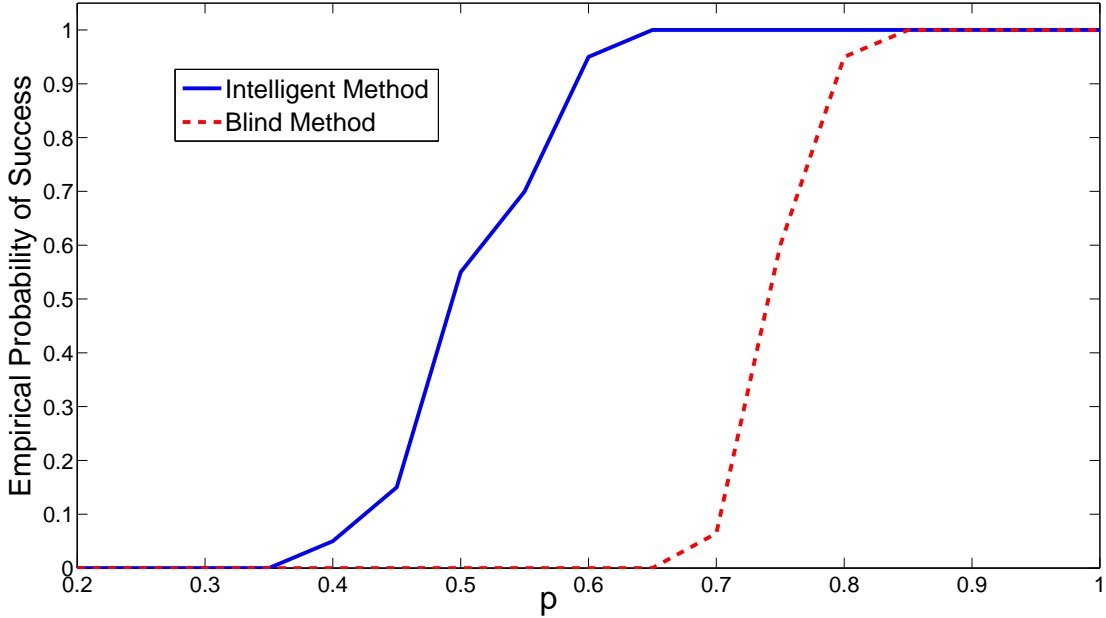}
 \caption{\scriptsize{Intelligent method succeeds for $p>0.6$ and blind succeeds for $p>0.8$. As we let $k_{min}\rightarrow \infty$ intelligent and blind methods will succeed for $p>q=0.1$ and $p>0.5$ respectively.}}
  \label{fig2}
\end{figure}

\subsection{Future Extensions}
\subsubsection{Alternative approaches}
Our simulation results indicate that a slight modification to problem (1) of \cite{Ames2} might be an alternative to the methods analyzed in this paper. Let $\e\in\R^n$ be the vector of all $1$'s. Then assuming we know the number of clusters $t$, proposed convex program is as follows
\begin{align}
\label{optim7}
&\max_{L}~\su(L_\Aa)\\
&\text{subject to}\\
&\hspace{15pt}L\succeq 0~~\text{(positive semi-definite)}\\
&\hspace{15pt}\text{trace}(L)=t\\
&\hspace{15pt}L_{i,j}\geq 0~~\text{for all}~1\leq i,j\leq n\\
&\hspace{15pt}(X\e)_i\leq e_i=1~~\text{for all}~1\leq i\leq n
\end{align}
The desired solution of this problem is
\beq
L^0=\sum_{i=1}^t\frac{1}{k_i}\1_{\Rr_{i,i}}
\eeq
It is easy to see that $L^0$ is feasible. Program (\ref{optim7}) might be a more useful approach compared to (\ref{optim2}) as it requires number of clusters $t$ as a prior information instead of $|\Rr|$. However, we only considered ``low rank $+$ sparse'' decompositions in this paper.
\subsubsection{Removing the disjointness assumption}
As a natural extension, we consider removing the assumption of disjoint clusters. When clusters are allowed to intersect, intuitively $\1_\Rr$ is no longer low rank. Although, we don't provide a proof, we believe rank of $\1_\Rr$ is equal to the number of distinct nonempty sets of type $\bigcap_{i\in S} \C_i\times \C_i$ where $S\subseteq[t]$. This suggests rank of $\1_\Rr$ can be as high as $2^t-1$ which grows exponentially with number of clusters. This intuition is verified by simulation results. Consequently, convex programs (\ref{optim2}) and (\ref{optim3}) might not be good candidates when clusters are allowed to intersect as we aim to find $\1_\Rr$ as a solution in these approaches. As a result, an alternative approach which will naturally result in a low rank solution is of significant interest. Another related problem is,  when clusters can intersect, how to obtain $\{\C_i\}_{i=1}^t$ from the knowledge of $\Rr$ assuming we are able to find $\Rr$ as a result of the optimization. Certainly, we may not always be able to uniquely decompose $\Rr$ into $\{\C_i\}$, but in general decomposition which yields smallest number of clusters might be of interest.
\subsubsection{Extremely sparse graph}
In many cases $\{p_i\},q$ decays as the model size grows. For example, in order to have a connected graph with high probability, Erd\"os-R\'enyi model with edge probability $r$ requires only $r>\frac{ln(n)}{n}$ i.e. average node degree of $ln(n)$. Sparse graphs are very common and useful in social networks \cite{Domino} and web graphs \cite{Kumar} hence it would be of interest to extend results of this paper to the setting where $\{p_i\},q$ are not constant. We believe this can be done by using concentration results specific to the spectral norm of sparse matrices.
%
%
%\begin{align}
%\label{optim5}
%&\max_{L,S}~\su(L_{\Aa^c})\\
%&\text{subject to}\\
%%&\hspace{15pt}L_{i,j}\geq 0\\
%&\hspace{15pt}\text{trace}(L)=t\\
%&\hspace{15pt}L\succeq 0^n~~(L~\text{is positive semidefinite})\\
%&\hspace{15pt}\|L\|\leq 1
%\end{align}
%Problem (\ref{optim5}) is in similar nature to (1) of \cite{Ames2}. The nice property of problem (\ref{optim5}) is that there is always a minimizer of rank exactly $k$. The difficulty with the A key property for (1) of \cite{Ames2}  is the fact that solution of the optimization, under right conditions, can be guessed
%Although, we won't discuss the exact reasons here, unlike \cite{Ames2}, in this program one cannot expect to have
%\beq
%L^*=\sum_{i=1}^t\frac{1}{k_i}\1_{\C_i\times \C_i}
%\eeq
%as the output. This is both due to the intersections and missing edges of the clusters. However, in order to be able to recover $\Rr$, we naturally would like $L^0$ to satisfy
%\beq
%\label{signhold}
%\text{sign}(L^*)=\1_{\Rr}
%\eeq
%In other words, although we won't know the exact values of entries of $L^*$, we'll hope that sign pattern of the solution gives the correct region and hence clusters as long as certain condition holds.

\subsection{Final Comments}
In this paper, we analyzed two novel approaches for detection of disjoint clusters in a general probabilistic model. Our results are consistent with the existing works in literature and significantly extend results of \cite{Ames}, \cite{Ames2}. Simulation results suggest that even for a relatively small model, our methods yield the desired result with high probability.

%Proof
%
%
%
%
%
%
%Proof of Intelligent
%
%
%
%
%
%
\section{Proof of Theorem \ref{mainthm}}
\label{intelproof}
Analysis of problems (\ref{optim2}) and (\ref{optim3}) are similar to a great extent. Therefore, many of the results for this section will also be used for section \ref{blindproof}. In the following discussion, $\la$ and $p_{min}-q$ is always assumed to be positive.

\subsection{Perturbation Analysis for $(L^0,S^0)$}
Let $(L^*,S^*)$ denote the optimal solution of problem (\ref{optim2}). We'll follow a conventional proof strategy to show that under some conditions, for any feasible nonzero perturbation $(E^L,E^S)$ over $(L^0,S^0)$ given in (\ref{optimal2}), the objective function strictly increases i.e.
\beq
\|L^0+E^L\|_\s+\lambda\|S^0+E^S\|_1> \|L^0\|_\s+\la\|S^0\|_1
\eeq
Consequently, due to convexity we'll conclude $(L^*,S^*)=(L^0,S^0)$.

\subsubsection{Observations}
\label{observes}
%Lemma 1
%
%
%
%
%
%
\begin{lem}
\label{lemma1}
For optimal solution of problem \ref{optim2}, we have $S^*=L^*_{\Aa^c}$.
\end{lem}
\begin{proof}
From (\ref{line2}) we have
\beq
\su((L^*-S^*)_{\Aa^c})=0
\eeq
This follows from the fact that $\1-\A=\1_{\Aa^c}$.
Combining this with (\ref{line1}), we can conclude that
\beq
\label{observe1}
L^*_{\Aa^c}=S^*_{\Aa^c}
\eeq
since $L^*_{i,j}\geq S^*_{i,j}$ for all $i,j\in[n]$. 
%Equation (\ref{observe1}) makes sense, as we want $L^*-S$ to correspond to the adjacency matrix of the subgraph induced by the clusters and if there is no edge between nodes $i,j$ then we should have $\A_{i,j}=(L-S)_{i,j}=0$.

Secondly, one can observe that if $(L,S)$ is feasible for problem (\ref{optim2}), then $(L,S_{\Aa^c})$ is also feasible and gives a lower (or equal) cost. This is because:
\begin{itemize}
\item The only constraint on entries of $S$ over $\Aa$ is $L_{i,j}\geq S_{i,j}\geq 0$ and $S_{i,j}=0$ will trivially satisfy this.
\item $\|S\|_1=\|S_\Aa\|_1+\|S_{\Aa^c}\|_1\geq \|S_{\Aa^c}\|_1$ with equality if and only if $S_\Aa=0$. Therefore, the objective will not increase by substituting $S$ by $S_{\Aa^c}$.
\end{itemize}
 Hence for optimality, we require:
\beq
\label{observe2}
S^*_{\Aa}=0
\eeq
Using (\ref{observe1}) and (\ref{observe2}), WLOG, $S$ takes the following simple form:
\beq
\label{Sfinal}
S^*=L^*_{\Aa^c}
\eeq
\end{proof}

A natural interpretation of \ref{Sfinal} is that, the only role of $S$ is filling the missing edges inside the clusters. Actually, we can write a simpler and equivalent optimization, where we get rid of the variable $S$; but still get the same result as problem (\ref{optim2}), as follows
\begin{align}
\label{optim4}
&\min_{L} \|L\|_\s+\lambda\|L_{\Aa^c}\|_1\\
\nonumber&\text{subject to}\\
%\nonumber&\hspace{15pt} S_{i,j}\geq 0~~\text{for all}~i,j\\
%\label{line1}
\nonumber&\hspace{15pt} 1\geq L_{i,j}\geq 0~~\text{for all}~i,j\\
\nonumber&\hspace{15pt} \su(L)\geq|\Rr|
\end{align}
Finally notice that $(L^0,S^0)$ satisfies (\ref{Sfinal}) as expected.

\subsubsection{Optimality Conditions for $(L^0,S^0)$}
Let $\li\cdot,\cdot\ri$ denote the usual inner product i.e. $\li X,Y\ri=\text{trace}(X^TY)=\sum_{i,j} X_{i,j}Y_{i,j}$. Also let $\text{sign}(\cdot):\R^{n\times n}\rightarrow\{-1,0,1\}^{n\times n}$ such that 
\beq
\text{sign}(X)_{i,j}=\begin{cases}1~~\text{if}~X_{i,j}>0\\0~~\text{if}~X_{i,j}=0\\-1~~\text{if}~X_{i,j}<0\end{cases}
\eeq
We would like to show any feasible nonzero perturbation $(E^L,E^S)$ over $(L^0,S^0)$ will strictly increase the objective. Due to Lemma \ref{lemma1}, we can assume 
\beq
\label{eform}
E^S=E^L_{\Aa^c}
\eeq
as $(L^0,S^0)$ satisfies (\ref{Sfinal}). In the following discussion, we analyze the increase in the objective due to the perturbation.

\noindent{\bf{Increase due to $E^S$:}} Similar to \cite{Candes}, by using the subgradient of the $\ell_1$ norm we can write:
\begin{equation}
\label{sparseq}
\|S^0+E^S\|_1\geq \|S^0\|_1+\li \text{sign}(S^0)+Q,E^S\ri
\end{equation}
for all $\|Q\|_\infty\leq 1$, $Q_{\Aa^c\cap \Rr}=0$ as $S^0$ is nonzero over $\Aa^c\cap \Rr$. Here $\|\cdot\|_\infty$ is the infinity norm, i.e., $\|X\|_\infty=\max_{1\leq i,j\leq n} |X_{i,j}|$.

Note that $\text{sign}(S^0)=S^0=\1_{\Aa^c\cap \Rr}$. Then, by choosing $Q=\1_{(\Aa^c\cap \Rr)^c}$ and using (\ref{eform}) in (\ref{sparseq}), we find:
\beq
\label{sparseq2}
\|S^0+E^S\|_1\geq \|S^0\|+\su(E^L_{\Aa^c})
\eeq

\noindent{\bf{Increase due to $E^L$:}} Let $\u_l\in\R^n$ be the characteristic vector of $\C_l$ with unit norm i.e. for $1\leq i\leq n$, $i$'th entry of $\u_l$ is
\beq
\u_{l,i}=\begin{cases} \frac{1}{\sqrt{k_l}}~~\text{if}~i\in\C_l\\
0~~\text{else}\end{cases}
\eeq
Let $\U=[\u_1~\dots~\u_t]\in\R^{n\times t}$ and $\MM_\U=\{X\in\R^{n\times n}:X\U=X^T\U=0\}$. Also $\|\cdot\|$ denotes the spectral norm, i.e., the maximum singular value. Then, following lemma characterizes the increase in the objective due to $E^L$.
\begin{lem} \label{lem11} For any $E^L$ and $W$ with $\|W\|\leq 1$, $W\in\MM_\U$ we have
\beq
\label{seven}
\|L^0+E^L\|_\s\geq \|L^0\|_\s+\sum_{l=1}^t \frac{1}{k_l}\su(E^L_{\Rr_{l,l}})+\left<E^L,W\right>
\eeq
\end{lem}
\begin{proof}
Singular value decomposition of $L^0$ can be written as 
\beq
\sum_{l=1}^tk_l\u_l\u_l^T=\U\begin{bmatrix}
k_1 & & & \\
& k_2 & &  \\
 & & \ddots & \\
& & & k_t
\end{bmatrix}\U^T
\eeq
as a result columns of $\U$, $\{\u_l\}_{l=1}^t$, are the left and right singular vectors of $L^0$. Then, we have
\beq
\|L^0+E^L\|_\s\geq \|L^0\|_\s+\left<E^L,W+\U\U^T\right>
\eeq
for any $W\in\MM_\U$ with $\|W\|\leq 1$, which follows from the subgradient of the nuclear norm, similar to \cite{Candes}. Finally, observe that
\beq
\U\U^T=\sum_{l=1}^t \frac{1}{k_l}\1_{\Rr_{l,l}}\implies\li E^L, \U\U^T\ri=\sum_{l=1}^t \frac{1}{k_l}\su(E^L_{\Rr_{l,l}})
\eeq
to conclude.
\end{proof}
{\bf{Overall increase:}} By combining (\ref{sparseq2}) and Lemma \ref{lem11}, we have the following lower bound for the increase of the objective:
\beq
\label{fincheck}
(\|L^0+E^L\|_\s-\|L^0\|_\s)+\la(\|S^0+E^S\|_1-\|S^0\|_1)\geq \sum_{l=1}^t\frac{1}{k_l}\su(E^L_{\Rr_{l,l}})+\la \su(E^L_{\A^c})+\li E^L,W\ri
\eeq
for any $W\in \MM_\U$, $\|W\|\leq 1$. Then, as long as the right hand side of (\ref{fincheck}) can be made strictly positive for all feasible nonzero $E^L$ (by properly choosing $W$), $(L^0,S^0)$ is the unique optimal solution of problem (\ref{optim2}). Let us call
\beq
\label{fel}
f(E^L,W)= \sum_{l=1}^t\frac{1}{k_l}\su(E^L_{\Rr_{l,l}})+\la \su(E^L_{\Aa^c})+\li E^L,W\ri
\eeq
\subsubsection{Main Cases}
The following lemma will help us separate the problem into two main cases.
\begin{lem}
\label{twocase}
Given $E^L$, assume there exists $W_0\in \MM_\U$ with $\|W_0\|<1$ such that $f(E^L,W_0)\geq 0$. Then at least one of the followings holds:
\begin{itemize}
\item There exists $W^*\in \MM_\U$ with $\|W^*\|\leq 1$ and $f(E^L,W^*)>0$
\item For all $W\in \MM_\U$, $\li E^L,W\ri=0$.
\end{itemize}
\end{lem}
\begin{proof}
Let $c=1-\|W_0\|$. Assume $\li E^L,W'\ri\neq 0$ for some $W'\in\MM_\U$. Since $\li E^L,W'\ri$ is linear in $W'$, WLOG, let $\li E^L,W'\ri>0$, $\|W'\|=1$. Then choose $W^*=W_0+cW'$. Clearly, $\|W^*\|\leq 1$, $W^*\in\MM_\U$ and 
\beq
f(E^L,W^*)=f(E^L,W_0)+\li E^L,cW'\ri>f(E^L,W_0)\geq 0
\eeq
\end{proof}
Notice that, for all $W\in \MM_\U$, $\li E^L,W\ri=0$ is equivalent to $E^L\in \MM_\U^\perp$ which is the orthogonal complement of $\MM_\U$ in $\R^{n\times n}$. $\MM_\U^\perp$ has the following simple characterization:
\beq
\MM_\U^\perp=\{X\in\R^{n\times n}:X=\U \M^T+\N\U^T~~\text{for some}~\M,\N\in\R^{n\times t}\}
\eeq
In the following discussion, based on Lemma \ref{twocase}, as a first step, in section \ref{seccase1}, we'll show that, under certain conditions, for all $E^L\in \MM_\U^\perp$ with high probability (w.h.p.)
\beq
\label{eqEL}
g(E^L)= \sum_{i=1}^t\frac{1}{k_l}\su(E^L_{\Rr_{l,l}})+\la \su(E^L_{\Aa^c})>0
\eeq
Secondly, in section \ref{seccase2}, we'll argue that, under certain conditions, there exists a $W\in \MM_\U$ with $\|W\|<1$ such that w.h.p. $f(E^L,W)\geq 0$ for all feasible $E^L$.  This $W$ is called the dual certificate. Finally, combining these two arguments, we'll conclude that $(L^0,S^0)$ is the unique optimal w.h.p.

\subsection{Solving for $E^L\in \MM_\U^\perp$ case}
\label{seccase1}
In order to simplify the following discussion, we let
\begin{align}
\label{gparts}
&g_1(X)=\sum_{i=1}^t\frac{1}{k_l}\su(X_{\Rr_{l,l}})\\
&\nonumber g_2(X)=\su(X_{\Aa^c})
\end{align}
so that $g(X)=g_1(X)+\la g_2(X)$ in (\ref{eqEL}). Also let $\V=[\v_1~\dots~\v_t]$ where $\v_i=\sqrt{k_i}\u_i$. Thus, $\V$ is basically obtained by, normalizing columns of $\U$ to make its nonzero entries $1$. 
Assume $E^L\in \MM_\U^\perp$. Then, we can write 
\beq
E^L=\V\M^T+\N\V^T
\eeq
Let $\m_i,\n_i$ denote $i$'th columns of $\M,\N$ respectively. Notice that $\su(L^0)=|\Rr|$ hence from (\ref{line3}) 
\beq
\label{observe3}
\su(E^L)\geq 0
\eeq 
Similarly, from $L^0$ and (\ref{line1}) it follows that
\begin{align}
\label{ES}
&E^L_{\Rr^c}~\text{is (entrywise) nonnegative}\\
&\nonumber E^L_{\Rr}~\text{is nonpositive}
\end{align}
Now, we list some simple observations regarding structure of $E^L$. We can write
\beq
\label{equa}
E^L=\sum_{i=1}^t (\v_i\m_i^T+ \n_i\v_i^T)=\sum_{i=1}^{t+1}\sum_{j=1}^{t+1}E^L_{\Rr_{i,j}}
\eeq

Notice that $E^L_{\Rr_{i,j}}$ is \emph{contributed} by only two components which are:  $\v_i\m_i^T$ and $\n_j\v_j^T$.

Let $\{a_{i,j}\}_{j=1}^{k_i}$ be an (arbitrary) indexing of elements of $\C_i$ i.e. $\C_i=\{a_{i,1},\dots,a_{i,k_i}\}$. For a vector $\z\in\R^n$ let $\z^i\in\R^{k_i}$ denote the vector induced by entries of $\z$ in $\C_i$. Basically, for any $1\leq j\leq k_i$, $\z^i_j=\z_{a_{i,j}}$. Also, let $E^{i,j}\in\R^{k_i\times k_j}$ which is $E^L$ induced by entries on $\Rr_{i,j}$. In other words, 
\beq
E^{i,j}_{c,d}=E^L_{a_{i,c},a_{j,d}}~~~\text{for any}~(i,j)\in\C_i\times \C_j~\text{and for any}~1\leq c\leq k_i,~1\leq d\leq k_j
\eeq
Basically, $E^{i,j}$ is same as $E^L_{\Rr_{i,j}}$ when we get rid of trivial zero rows and zero columns.  Then
\beq
\label{eij}
E^{i,j}=\one^{k_i}{\m_i^j}^T+\n_j^i{\one^{k_j}}^T
\eeq
Clearly, given $\{E^{i,j}\}_{1\leq i,j\leq n}$, $E^L$ is uniquely determined. Now, assume we fix $\su(E^{i,j})$ for all $i,j$ and we would like to find the \emph{worst} $E^L$ subject to these constraints. Variables in such an optimization are $\m_i,\n_i$. Basically we are interested in
\begin{align}
\label{optimel}
&\min g(E^L)\\
&\text{subject to}\\
&\hspace{15pt}\su(E^{i,j})=c_{i,j}~\text{for all}~i,j\\
\label{ES2}
&\hspace{15pt}E^{i,j}~\begin{cases}\text{nonnegative if}~i\neq j\\\text{nonpositive if}~i=j\end{cases}
\end{align}
where $\{c_{i,j}\}$ are constants. Constraint (\ref{ES2}) follows from  (\ref{ES}). Essentially, based on (\ref{observe3}), we would like to show that with high probability for any nonzero $E^L$ with $\sum_{i,j}c_{i,j}\geq 0$ we have $g(E^L)>0$. {{\bf{Remark:}}} For the special case of $i=j=t+1$, notice that $E^{i,j}=0$.

In (\ref{optimel}), $g_1(E^L)$ is fixed and equal to $\sum_{i=1}^t \frac{1}{k_i}c_{i,i}$. Consequently, based on (\ref{gparts}), we just need to do the optimization with the objective $g_2(E^L)=\su(E^L_{\Aa^c})$.

Let $\beta_{i,j}\subseteq [k_i]\times [k_j]$ be a set of coordinates defined as follows. For any $(c,d)\in [k_i]\times [k_j]$
\beq
(c,d)\in\beta_{i,j}~\text{iff}~(a_{i,c},a_{j,d})\in\Aa
\eeq
%$\beta^c_{i,j}$ is defined in the exact same way.

For $(i_1,j_1)\neq (i_2,j_2)$, $(\m_{i_1}^{j_1},\n_{j_1}^{i_1})$ and $(\m_{i_2}^{j_2},\n_{j_2}^{i_2})$ are independent variables. Consequently, due to (\ref{eij}), we can partition problem (\ref{optimel}) into the following smaller disjoint problems.
\begin{align}
\label{local}
&\min_{\m_i^j,\n_j^i}~ \su(E^{i,j}_{\beta^c_{i,j}})\\
&\text{subject to}\\
&\hspace{15pt}\su(E^{i,j})=c_{i,j}\\
&\hspace{15pt}E^{i,j}~\begin{cases}\text{nonnegative if}~i\neq j\\\text{nonpositive if}~i= j
\end{cases}
\end{align}
Then, we can solve these problems locally (for each $i,j$) to finally obtain
\beq
g_2(E^{L,*})=\sum_{i,j}\su(E^{i,j,*}_{\beta^c_{i,j}})
\eeq
to find the overall result of problem (\ref{optimel}), where $*$ denotes the optimal solutions in problems (\ref{optimel}) and (\ref{local}). The following lemma will be useful for analysis of these local optimizations.
\begin{lem}
\label{useful}
Let $\aaa\in\R^{c}$, $\bbb\in\R^{d}$ and $X=\one^{c}\bbb^T+\aaa{\one^{d}}^T$ be variables and $C_0\geq 0$ be a constant. Also let $\beta\subseteq[c]\times [d]$. Consider the following optimization problem
\begin{align}
&\min_{\aaa,\bbb}~\su(X_\beta)\\
&\text{subject to}\\
&\hspace{15pt}X_{i,j}\geq 0~~\text{for all}~i,j\\
&\hspace{15pt}\su(X)=C_0
\end{align}
For this problem there exists a (entrywise) nonnegative minimizer $(\aaa^0,\bbb^0)$.
\end{lem}
\begin{proof}
Let $x_i$ denotes $i$'th entry of vector $\x$. Assume $(\aaa^*,\bbb^*)$ is a minimizer. WLOG assume $b^*_1=\min_{i,j}\{\aaa^*_i,\bbb^*_j\}$. If $b^*_1\geq 0$ we are done. Otherwise, since $X_{i,j}\geq 0$ we have $a^*_i\geq -b^*_1$ for all $i\leq c$. Then set $\aaa^0=\aaa^*+\one^c b^*_1$ and $\bbb^0=\bbb^*-\one^db^*_1$. Clearly, $(\aaa^0,\bbb^0)$ is nonnegative. On the other hand, we have:
\beq
X^*=\one^{c}{\bbb^*}^T+\aaa^*{\one^{d}}^T=\one^{c}{\bbb^0}^T+\aaa^0{\one^{d}}^T=X^0\implies \su(X^*_\beta)=\su(X^0_\beta)=\text{minimum value}
\eeq
\end{proof}
\begin{lem}
\label{direct}
A direct consequence of Lemma \ref{useful} is the fact that in the local optimizations (\ref{local}), WLOG we can assume $(\m_i^j,\n_j^i)$ entrywise nonnegative whenever $i\neq j$ and entrywise nonpositive when $i=j$. This follows from the structure of $E^{i,j}$ given in (\ref{eij}) and (\ref{ES}).
\end{lem}
 Following lemma will help us characterize the relationship between $\su(E^{i,j})$ and $\su(E^{i,j}_{\beta^c_{i,j}})$.
\begin{lem}
\label{lemprob}
Let $\beta\in\R^{c\times d}$ be a random support with parameter $0\leq r\leq 1$. Then for any $\eps>0$ w.p.a.l. $1-d\exp(-2\eps^2 c)$ for all nonzero and entrywise nonnegative $\aaa\in\R^d$ we'll have:
\beq
\label{probstate1}
\su(X_\beta)>(r-\eps)\su(X)
\eeq
where $X=\one^c\aaa^T$. Similarly, with same probability, for all such $\aaa$, we'll have $\su(X_\beta)<(r+\eps)\su(X)$
\end{lem}
\begin{proof}
We'll only prove the first statement (\ref{probstate1}) as proofs are identical. For each $i\leq d$, $a_i$ occurs exactly $c$ times in $X$ as $i$'th column of $X$ is $\one^ca_i$. By using a Chernoff bound, we can estimate the number of coordinates of $i$'th column which are element of $\beta$ (call this number $C_i$) as we can view this number as a sum of $c$ i.i.d. $\text{Bern}(1,0,r)$ random variables. Then
\beq
\Prob(C_i\leq c(r-\eps))\leq\exp(-2\eps^2 c)
\eeq
Now, we can use a union bound over all columns to make sure for all $i$, $C_i> c(r-\eps)$
\beq
\Prob(C_i>c(r-\eps)~\text{for all}~i\leq d)\geq 1-d\exp(-2\eps^2 c)
\eeq
On the other hand if each $C_i> c(r-\eps)$ then for any nonnegative $\aaa\neq 0$
\beq
\su(X_\beta)=\sum_{(i,j)\in\beta}X_{i,j}=\sum_{i=1}^d C_i a_i> c(r-\eps)\sum_{i=1}^da_i=(r-\eps)\su(X)
\eeq
\end{proof}

Using Lemma \ref{lemprob}, we can calculate a lower bound for $g(E^L)$ with high probability as long as cluster sizes are sufficiently large. Due to (\ref{equa}) and the linearity of $g(E^L)$, we can focus on contributions due to specific clusters i.e. $\v_i\m_i^T+ \n_i\v_i^T$ for the $i$'th cluster. We additionally know the simple structure of $\m_i,\n_i$ from Lemma \ref{direct}. In particular, subvectors $\m_i^i$ and $\n_i^i$ of $\m_i,\n_i$ can be assumed to be nonpositive and rest of the entries are nonnegative.

Now, we define an important parameter which will be useful for subsequent analysis. This parameter can be seen as a measure of distinctness of the ``worst'' cluster from the ``background noise''. Here background noise corresponds to the edges over $\Rr^c$.
\beq
\label{eeq}
e=\min_{l\leq t}\frac{1}{2}(p_l-q-\frac{1}{k_l\la})
\eeq

 The following lemma, gives a lower bound on $g(\v_l\m_l^T)$.

\begin{lem}
\label{cool}
Assume $e>0$. Then, w.p.a.l. $1-n\exp(-2e^2(k_l-1))$, we have $g(\v_l\m_l^T)\geq \la (1-q-e)\su(\v_l\m_l^T)$ for all $\m_l$. Also, if $\m_l\neq 0$ then inequality is strict.
\end{lem}
\begin{proof}
Let us call $X^i=\one^{k_l}{\m_l^i}^T$. Also $\m_l^i$ is nonnegative for $i\neq l$ and nonpositive for $i=l$. Then
\begin{align}
g(\v_l\m_l^T)&=\frac{1}{k_l}\su((\v_l\m_l^T)_{\Rr_{l,l}})+\la \su((\v_l\m_l^T)_{\Aa^c})\\
&=\frac{1}{k_l}\su(\one^{k_l}{\m^l_l}^T)+\sum_{i=1}^t\la\su((\one^{k_l}{\m_l^i}^T)_{\beta^c_{l,i}})\\
\label{coolalign}
&=\frac{1}{k_l}\su(X^l)+\sum_{i=1}^t\la\su(X^i_{\beta^c_{l,i}})
\end{align}
$\beta_{l,i}$ is a random support with parameter $q$ if $i\neq l$ and corrected random support with parameter $p$ if $i=l$. For a fixed $i\leq t+1$, from Lemma \ref{lemprob} w.p.a.l. $1-k_i\exp(-2e^2(k_l-1))$ we have
\beq
\label{eqnapp1}
\su(X^i_{\beta^c_{l,i}})\geq\begin{cases}(1-q-e)\su(X^i)~\text{if}~i\neq l\\(1-p_l+e)\su(X^i)~\text{if}~i=l\end{cases}
\eeq
Then, using a union bound w.p.a.l. $1-n\exp(-2e^2(k_l-1))$ we have (\ref{eqnapp1}) for all $i$ and $\m_l$. Combining this with (\ref{coolalign}), we get
\begin{align}
\label{strictapp1}
g(\v_l\m_l^T)&\geq\la\sum_{i\neq l}(1-q-e)\su(X^i)+\left(\frac{1}{k_l}+\la(1-p_l+e)\right)\su(X^l)\\
&\geq \la (1-q-e)\sum_{i=1}^{t+1}\su(X^i)=\la (1-q-e)\su(\v_l\m_l^T)
\end{align}
If $\m_l\neq 0$, inequality (\ref{eqnapp1}) is strict for some $1\leq i\leq t+1$ due to Lemma \ref{lemprob}. Hence, (\ref{strictapp1}) will be strict too.
\end{proof}

As we have mentioned in section \ref{mainresult}, let $k_{min}$ denote the size of the minimum cluster, which will be an important parameter for rest of our analysis. Following theorem is based on Lemma \ref{cool} and gives the main result of this section.
\begin{thm}
\label{part1}
Let $e$ be same as described in (\ref{eeq}). Assume $\la$ and $\{k_i\}$ are such that $e>0$. Then w.p.a.l. $1-2nt\exp(-2e^2(k_{min}-1))$, for any $E^L\neq 0$ with $E^L\in\MM_\U^\perp$ and $\su(E^L)\geq 0$ we have $g(E^L)>0$.
%In particular, keeping $\la$ constant and letting $n,k_{min}$ grow to infinity, we have $P_{success}\rightarrow 1$.
\end{thm}
\begin{proof} 
Due to Lemma \ref{cool}, for a particular $l$, w.p.a.l. $P_l=1-n\exp(-2e^2(k_l-1))$ we have
 \beq
 \label{lllabel}
 g(\v_l\m_l^T)\geq \la (1-q-e)\su(\v_l\m_l^T)
 \eeq
 and an identical result holds for $\n_l\v_l^T$ term.
 
 Now union bounding over all $\{\m_l\},\{\n_l\}$, we can obtain w.p.a.l.
 \beq
1-2nt\exp(-2e^2(k_{min}-1))\leq 1-2\sum_{i=1}^t(1-P_l)
 \eeq for all $l\leq t$ (\ref{lllabel}) holds, hence going back to (\ref{equa}) 
 \begin{align}
 g(E^L)&=\sum_{i=1}^tg(\v_l\m_l^T+\n_l\v_l^T)\\
 \label{ll4}
 &\geq \la (1-q-e)[\su(\v_l\m_l^T)+\su(\n_l\v_l^T)]\\
 &=\la (1-q-e) \su(E^L)\geq 0
 \end{align}
 On the other hand, if $E^L\neq 0$ then at least one of $\{\m_l\},\{\n_l\}$ is nonzero and inequality (\ref{ll4}) is actually strict.
\end{proof}
Hence, the main result of this section is the fact that, as long as $\la$ and the cluster sizes $\{k_i\}$ are sufficiently large, we don't need to worry about feasible perturbations of type $E^L\in\MM_\U^\perp$.

%Dual certificate
%
%
%
%
%
\subsection{Showing existence of the dual certificate}
\label{seccase2}
In this section, we'll treat the second case. Our aim is showing the existence of a $W\in \MM_\U$ with $\|W\|<1$ such that $f(E^L,W)\geq 0$ for all feasible $E^L$. We follow an approach consisting of three steps:
\begin{itemize}
\item Construct a candidate $W_0$ which satisfies $f(E^L,W_0)\geq 0$ for all feasible $E^L$.
\item Show that $\|W_0\|<1$ under certain conditions.
\item \emph{Slightly} modify $W_0$ to obtain $W$ which still satisfies the previous conditions, but also obeys $W\in\MM_\U$.
\end{itemize}
\subsubsection{Candidate $W_0$}
%First, using (\ref{ES}) and (\ref{fel}) we observe that
%\beq
%f(E^L,W)\geq \frac{1}{k_{min}}E^L_{\Rr}+\la E^L_{\Aa^c}+\li E^L,W\ri
%\eeq
Recall that \beq
%\label{fel}
f(E^L,W)= \sum_{l=1}^t\frac{1}{k_l}\su(E^L_{\Rr_{l,l}})+\la \su(E^L_{\Aa^c})+\li E^L,W\ri
\eeq
Using approaches similar to \cite{Ames} and \cite{Candes}, we'll construct a $W$ based on the following candidate
\beq
\label{W_0}
W_0=c\1+\la \1_{\Aa}+\sum_{i=1}^tc_i \1_{\Rr_{i,i}}
\eeq
Here $c,\{c_i\}_{i=1}^t$ are the variables that will be used to construct the desired $W$. Now, let us see, why this $W_0$ is an \emph{intelligent} choice
\beq
f(E^L,W_0)=\sum_{i=1}^t (c_i+\frac{1}{k_i})\su(E^L_{\Rr_{i,i}})+(\la+c)\su(E^L)
\eeq
Notice that if $c_i+\frac{1}{k_{i}}\leq 0$ and $\la+c\geq 0$ we are done since $\su(E^L)\geq 0 $ and $\su(E^L_{\Rr_{i,i}})\leq 0$ for all $i$. Obviously, one needs to do this, while ensuring $\|W_0\|$ is as \emph{small} as possible.

In (\ref{W_0}), for constant $c,\{c_i\},\la$, $W_0$ is a random matrix with i.i.d. entries due to the $\1_\Aa$ term as graph is randomly generated. An intuitive way of ensuring small $\|W_0\|$ is to force expectation of $W_0$ to $0$. In order to ensure the expectation of the entries inside the region $\Rr^c$ is $0$ we need
\beq
(\la+c)q+c(1-q)=0
\eeq
Hence $c=-\la q$. Now setting expectation over $\Rr_{i,i}$ to $0$, we find
\beq
(c_i+\la+c)p_i+(c_i+c)(1-p_i)=0
\eeq
Hence 
\beq
c_i=-c(1-p_i)-(c+\la)p_i=\la q(1-p)-\la(1-q)p=\la (q-p_i)
\eeq
Now notice that we satisfy $\la+c=\la(1-q)\geq 0$. In order to satisfy $c_i+\frac{1}{k_i}\leq 0$ we need
\beq
\label{labelpq}
2e=\min_{i\leq t}\left[p_i-q-\frac{1}{\la k_i}\right]\geq 0
\eeq
The reader will remember that this is the same constraint we needed for Theorem \ref{part1}. With these choices of $\{c_i\},c$ we have
\begin{align}
W_0&=\la (-q\1+\1_{\Aa}+\sum_{i=1}^t(q-p_i) \1_{\Rr_{i,i}})\\
&=\la (\sum_{i=1}^t[(1-p_i)\1_{\Aa\cap \Rr_{i,i}}-p_i\1_{\Aa^c\cap\Rr_{i,i}}]+[(1-q)\1_{\Aa\cap\Rr^c}-q\1_{\Aa^c\cap\Rr^c}])
\end{align}
and
\beq
\label{ELB}
f(E^L,W_0)= \la(1-q)\su(E^L) -\sum_{i=1}^t(\la(p_i-q)-\frac{1}{k_i})\su(E^L_{\Rr_{i,i}})
\eeq
Assuming (\ref{labelpq}), since $\su(E^L)\geq 0$ and $\su(E^L_{\Rr_{i,i}})\leq 0$, for any feasible $E^L$, $f(E^L,W_0)\geq 0$, thus $W_0$ is indeed a good choice. However, there are two problems to be solved.
\begin{itemize}
\item Making sure that $\|W_0\|$ is sufficiently small.
\item ``Correcting'' $W_0$ so that $W_0\in \MM_\U$ while still ensuring $f(E^L,W_0)\geq 0$ for all $E^L$.
\end{itemize}
\subsubsection{Bounding the spectral norm}
Following lemma addresses the first problem and gives a simple bound on $\|W_0\|$.

\begin{lem}
\label{sondanbir}
Recall that $W_0$ is a random matrix where randomness is on $\Aa$ and $W_0$ is given by
\beq
\label{w0form}
W_0=\la (\sum_{i=1}^t[(1-p_i)\1_{\Aa\cap \Rr_{i,i}}-p_i\1_{\Aa^c\cap\Rr_{i,i}}]+[(1-q)\1_{\Aa\cap\Rr^c}-q\1_{\Aa^c\cap\Rr^c}])
\eeq
Then, for any $\eps>0$, w.p.a.l. $1-4\exp(-\eps^2 \frac{n}{32})$ we have
\beq
\|W_0\|\leq (1+\eps+o(1))\la \sqrt{n}
\eeq
\end{lem}

\begin{proof}
$\frac{1}{\la}W_0$ is a random matrix whose entries are i.i.d. and distributed as $\text{Bern}(-p_i,1-p_i,1-p_i)$ on $\Rr_{i,i}$ and $\text{Bern}(-q,1-q,1-q)$ on $\Rr^c$. Then variance of an entry is at most $\max\{\{p_i(1-p_i)\}_{i=1}^t,q(1-q)\}\leq 1/4$ hence we can use Theorem 1.5 of \cite{VHVu} to find
\beq
\text{median}(\|\frac{1}{\la}W_0\|)\leq (2\sqrt{\max\{\{p_i(1-p_i)\},q(1-q)\}}+o(1))\sqrt{n}\leq (1+o(1))\sqrt{n}
\eeq
On the other hand, since absolute values of entries of $\frac{1}{\la}W_0$ are bounded by $1$, Theorem 1 of \cite{Alon} gives
\beq
4\exp(-\eps^2\frac{n}{32})\geq \Prob\left[\|\frac{1}{\la}W_0\|>\text{median}(\|\frac{1}{\la}W_0\|)+\eps\sqrt{n}\right]\geq \Prob\left[\|W_0\|>\la(1+\eps+o(1))\sqrt{n}\right]
\eeq
\end{proof}

Lemma \ref{sondanbir} verifies that asymptotically with high probability we can make $\|W_0\|<1$ as long as we choose a proper $\la$ which yields sufficiently small $\la\sqrt{n}$. However, $W_0$ itself is not sufficient for construction of the desired $W$, since we don't have any guarantee that $W_0\in\MM_\U$. In order to achieve this, we'll \emph{correct} $W_0$ by projecting it onto $\MM_\U$. Following lemma suggests that we don't lose much by such a correction.
\subsubsection{Correcting the candidate $W_0$}
\begin{lem}
\label{correction}
$W_0$ is as described previously in (\ref{w0form}). Let $W^H$ be projection of $W_0$ on $\MM_\U$. Then
\begin{itemize}
\item $\|W^H\|\leq \|W_0\|$
\item For any $\eps>0$, w.p.a.l. $1-6n^2\exp(-2\eps^2k_{min})$ we have
\beq
\|W_0-W^H\|_\infty\leq 3\la\eps
\eeq
\end{itemize}

\end{lem}
\begin{proof}
Choose arbitrary vectors $\{\u_i\}_{i=t+1}^n$ to make $\{\u_i\}_{i=1}^n$ an orthonormal basis in $\R^n$. Call $\U_2=[\u_{t+1}~\dots~\u_n]$ and $\PPP=\U\U^T$, $\PPP_2=\U_2\U_2^T$. Now notice that for any matrix $\X\in\R^{n\times n}$, $\PPP_2\X\PPP_2$ is in $\MM_\U$ since $\U^T\U_2=0$. Let ${\bf{I}}$ denote the identity matrix. Then
\beq
\label{eqstupid}
\X-\PPP_2\X\PPP_2=\X-({\bf{I}}-\PPP)\X({\bf{I}}-\PPP)=\PPP\X+\X\PPP-\PPP\X\PPP\in\MM_\U^\perp
\eeq
Hence, $\PPP_2\X\PPP_2$ is the orthogonal projection on $\MM_\U$. Clearly
\beq
\|W^H\|=\|\PPP_2W_0\PPP_2\|\leq \|\PPP_2\|^2\|W_0\|\leq \|W_0\|
\eeq

For analysis of $\|W_0-W^H\|_\infty$ we can consider terms on right hand side of (\ref{eqstupid}) separately as we have:
\beq
\|W_0-W^H\|_\infty\leq \|\PPP W_0\|_\infty+ \| W_0\PPP\|_\infty+ \|\PPP W_0\PPP\|_\infty
\eeq
Clearly $\PPP=\sum_{i=1}^t \frac{1}{k_l}\1_{\R_{l,l}}$. Then, each entry of $\frac{1}{\la}\PPP W_0$ is either a summation of $k_i$ i.i.d. $\text{Bern}(-p_i,1-p_i,1-p_i)$ or $\text{Bern}(-q,1-q,1-q)$ random variables scaled by $k_i^{-1}$ for some $i\leq t$ or $0$. Hence any $c,d\in[n]$ and $\eps>0$
\beq
\Prob[|(\PPP W_0)_{c,d}|\geq\la\eps] \leq 2\exp(-2\eps^2 k_{min})
\eeq
Same (or better) bounds holds for entries of $W_0\PPP$ and $\PPP W_0\PPP$. Then a union bound over all entries of the three matrices will give w.p.a.l. $1-6n^2\exp(-2\eps^2 k_{min})$, we have $\|W_0-W^H\|_\infty\leq 3\la\eps$. \end{proof}

\subsubsection{Summary of section \ref{seccase2}}
Lemma \ref{correction} suggests that actually $W_0$ can be corrected with an arbitrarily small perturbation. This will be useful in the following theorem which summarizes main result of this section.
\begin{thm}
\label{part2}
$W_0$ and $e$ are as described previously in (\ref{w0form}), (\ref{eeq}) respectively. Choose $W$ to be projection of $W_0$ on $\MM_\U$. Also set $\la=\frac{1}{2\sqrt{n}}$ and assume $\{k_i\}_{i=1}^t$ is such that $e>0$.

Then, w.p.a.l. $1-6n^2\exp(-\frac{2}{9}e^2k_{min})-4\exp(-\frac{n}{100})$ we have
\begin{itemize}
\item $\|W\|<1$
\item For all feasible $E^L$, $f(E^L,W)\geq 0$. 
\end{itemize}
\end{thm}
\begin{proof}
First consider Lemma \ref{sondanbir}. Let $\eps=\frac{\sqrt{32}}{10}$. Then w.p.a.l. $1-4\exp(-\frac{n}{100})$ we have 
\beq
\|W\|\leq \|W_0\|\leq (1+\eps+o(1))\la\sqrt{n} <1
\eeq

Now, assume, we have $\|W_0-W\|_\infty= \la\eps_0$. Then, using (\ref{ES}), for any $E^L$, we can write
\begin{align}
\li W_0-W,E^L\ri&\leq \la\eps_0(\su(E^L_{\Rr^c})- \su(E^L_\Rr))\\
&=\la\eps_0(\su(E^L)-2\su(E^L_{\Rr}))
\end{align}
Now, we consider (\ref{ELB}). As long as $\eps_0\leq\min\{1-q,e\}=e$, for any feasible $E^L$ we have
\begin{align}
f(E^L,W)&=f(E^L,W_0)-\li W_0-W,E^L\ri\geq f(E^L,W_0)-\la\eps_0(\su(E^L)-2\su(E^L_{\Rr}))\\
&=\la[(1-q-\eps_0)\su(E^L)-\sum_{i=1}^t(p_i-q-\frac{1}{\la k_{i}}-2\eps_0)\su(E^L_{\Rr_{i,i}})]\geq 0
\end{align}
Hence $W$ satisfies the desired condition. Lemma \ref{correction} gives the following concentration for $\|W_0-W\|_\infty$
\beq
\Prob[\|W_0-W\|_\infty> \la e]\leq 6n^2\exp(-\frac{2}{9}e^2k_{min})
\eeq
Finally, a union bound over the failure of events $\|W\|<1$ and $\|W_0-W\|_\infty\leq \la e$ gives the result.
\end{proof}
Theorem \ref{part2} concludes this section because our aim throughout the section was constructing such a $W$ w.h.p. As a final step, we combine, Theorems \ref{part1} and \ref{part2} and Lemma \ref{twocase} to deduce the main result for the intelligent approach.

\subsection{Final step}
Following theorem finishes proof of Theorem \ref{mainthm} by combining Theorems \ref{part1} and \ref{part2}.
\begin{proof}[{\bf{Proof of Theorem \ref{mainthm}}}]
For the following discussion $C_1,C_2,C_3>0$ are the suitable constants for the previous theorems. Let $e$ be same as before. Then $e\geq \min_{i\leq t}\frac{p_i-q}{4}$ and statements of Theorem \ref{part1} will hold w.p.a.l. $1-2nt\exp(-C_1(p_{min}-q)^2k_{min})$ and $1-6n^2\exp(-C_2(p_{min}-q)^2k_{min})-4\exp(-C_3 n)$ respectively. Then using a union bound and $n\geq k_{i}$ both statements hold with w.p.a.l. $1-(8n^2+o(1))\exp(-\min\{C_1,C_2,C_3\}(p_{min}-q)^2k_{min})$ and we have
\begin{itemize}
\item From Theorem \ref{part1}, for any nonzero $E^L\in\MM_\U^\perp$, $g(E^L)=f(E^L,0)>0$ hence objective increases.
\item Otherwise, due to Theorem \ref{part2}, there exists a $\|W\|<1$, $W\in\MM_\U$ such that for any $E^L$, $f(E^L,W)\geq 0$. Then, from Lemma \ref{twocase}, there exists $W^*\in\MM_\U$ with $\|W^*\|\leq 1$ such that $f(E^L,W^*)>0$ hence objective increases.
\end{itemize}
Then for all $E^L$, objective increases which implies $(L_0,S_0)$ is the unique optimal solution of problem \ref{optim2}.
\end{proof}

%Proof of Blind
%
%
%
%
%
%
\section{Proof of Theorem \ref{mainthmb}}
\label{blindproof}
We'll follow almost the same approach and notation in section \ref{intelproof}. We aim to show $(L^0,S^0)$ given in (\ref{optimal}) is unique optimal to problem \ref{optim3}. 
\subsection{Perturbation analysis}
\begin{lem}
Let $(E^L,E^S)$ be a feasible perturbation. Then, objective will increase by at least
\beq
f(E^L,W)= \sum_{l=1}^t\frac{1}{k_l}\su(E^L_{\Rr_{l,l}})+\li E^L,W\ri+\la (\su(E^L_{\Aa^c})-\su(E^L_{\Aa}))
\eeq
for any $W\in\MM_\U$, $\|W\|\leq1$.
\end{lem}
\begin{proof}
Clearly $E^L=-E^S$ as $L^0+S^0=\A$. Similar to previous section, for any such $W$ increase in $\|L\|_\s$ satisfies
\beq
\|L^0+E^L\|_\s-\|L^0\|_\s=\sum_{l=1}^t\frac{1}{k_l}\su(E^L_{\Rr_{l,l}})+\li E^L,W\ri
\eeq
For sparse component, using $\text{sign}(S^0)=\1_{\Aa\cap\Rr^c}-\1_{\Aa^c\cap\Rr}$ and choosing $Q=\1_{\Aa\cap\Rr}-\1_{\Aa^c\cap \Rr^c}$ we find:
\beq
\|S^0-E^L\|_1-\|S^0\|_1\geq\left< -E^L,\text{sign}(S^0)+Q\right>=\su(E^L_{\Aa^c})-\su(E^L_\Aa)
\eeq
Combining these, we get the desired form $f(E^L,W)$.
\end{proof}
Notice that we can directly use Lemma \ref{twocase}. Let
\beq
g(E^L)= \sum_{l=1}^t\frac{1}{k_l}\su(E^L_{\Rr_{l,l}})+\la (\su(E^L_{\Aa^c})-\su(E^L_{\Aa}))
\eeq
Then, we first show w.h.p. objective strictly increases for all $E^L\in\MM_\U^\perp$ and then w.h.p. construct a dual certificate $W$ satisfying $\|W\|<1$, $W\in\MM_\U$ and for all feasible $E^L$, $f(E^L,W)\geq 0$.

\subsection{Solving for $E^L\in \MM_\U^\perp$ case}
\label{seccase3}
Let $g_1(X)=\sum_{l=1}^t\frac{1}{k_l}\su(X_{\Rr_{l,l}})$ and $g_2(X)=\su(X_{\Aa^c})-\su(X_{\Aa})$.

\subsubsection{Summary of the similarities with the proof of Theorem \ref{mainthm}}
$E^L$ has the form $\V\M^T+\N\V^T$ and $\M,\N,\{\m_i\},\{\n_i\},\{\beta_{i,j}\},\{a_{i,j}\}$ are as described in section \ref{intelproof}. Again we consider, problem \ref{optimel} and since $g_1(E^L)$ is fixed, we just need to optimize over $g_2(E^L)$. This optimizations can be reduced to local optimizations \ref{local}. Since $L^0=\1_\Rr$, (\ref{ES}) applies for $E^L$ and we can make use of Lemma \ref{direct} and assume $\m_l^i$ is nonpositive/nonnegative when $i=l/i\neq l$ for all $i,l$. Hence, using Lemma \ref{lemprob} we lower bound $g(\v_l\m_l^T)$ as follows.
\subsubsection{Lower bounding $g(E^L)$}
For the purpose of this section, we set $e$ as follows:
\beq
\label{eeq2}
e=\frac{1}{2}\min\{1-2q,\{2p_l-\frac{1}{\la k_l}-1\}_{l=1}^t\}
\eeq
\begin{lem}
\label{cool2}
Assume, $l\leq t$, $e>0$. Then, w.p.a.l. $1-n\exp(-2e^2(k_l-1))$, we have $g(\v_l\m_l^T)\geq 0$ for all $\m_l$. Also, if $\m_l\neq 0$ then inequality is strict.
\end{lem}
\begin{proof}
Recall that $\m_l$ satisfies $\m_l^i$ is nonpositive/nonnegative when $i=l/i\neq l$ for all $i$. Call $X^i=\one^{k_l}{\m_l^i}^T$. We can write
\beq
g(\v_l\m_l^T)=\frac{1}{k_l}\su(X^l)+\sum_{i=1}^t\la h(X^i,\beta^c_{l,i})
\eeq
where $h(X^i,\beta^c_{l,i})=\su(X^i_{\beta^c_{l,i}})-\su(X^i_{\beta_{l,i}})$. Now assume $i\neq l$. Using Lemma \ref{lemprob} and the fact that $\beta_{l.i}$ is a random support with $q$ w.p.a.l. $1-k_i\exp(-2\eps^2k_l)$, for all $X^i$, we have
\beq
h(X^i,\beta^c_{l,i})\geq (1-q-\eps)\su(X^i)-(q+\eps)\su(X^i)=(1-2q-2\eps)\su(X^i)
\eeq
where inequality is strict if $X^i\neq 0$. Similarly when $i=l$ we have w.p.a.l. $1-k_l\exp(-2\eps^2(k_l-1))$
\beq
\frac{1}{\la k_l}\su(X^l)+h(X^l,\beta^c_{l,l})\geq (1-p_l+\eps+\frac{1}{\la k_l})\su(X^l)-(p_l-\eps)\su(X^l)=-(2p_l-1-\frac{1}{\la k_l}-2\eps)\su(X^l)
\eeq
Choosing $\eps=e$ and using the facts $1-2q-2e\geq 0$, $2p_l-1-\frac{1}{\la k_l}-2e\geq 0$ and a union bound w.p.a.l. $1-n\exp(-2e^2 (k_l-1))$ we have $g(\v_l\m_l^T)\geq 0$ and inequality is strict when $\m_l\neq 0$ as at least one of the $X^i$'s will be nonzero.
\end{proof}
Following theorem immediately follows from Lemma \ref{cool2} and summarizes the main result of the section.
\begin{thm}
\label{part3}
Let $e$ be as in (\ref{eeq2}) and assume $e>0$. Then w.p.a.l. $1-2nt\exp(-2e^2(k_{min}-1))$ we have $g(E^L)>0$ for all nonzero feasible $E^L\in\MM_U^\perp$.
\end{thm}

\subsection{Showing existence of the dual certificate}
\label{seccase4}
Again, we'll follow quite similar steps to section \ref{seccase2}. Recall that
\beq
f(E^L,W)= \sum_{i=1}^t\frac{1}{k_i}\su(E^L_{\Rr_{i,i}})+\li E^L,W\ri+\la(\su(E^L_{\Aa^c})-\su(E^L_{\Aa}))
\eeq
$W$ will be constructed from the candidate $W_0$ as follows.
\subsubsection{Candidate $W_0$}
Based on convex program \ref{optim3}, we propose the following form
\beq
W_0=\sum_{i=1}^tc_i\1_{\Rr_{i,i}}+c\1_{\Rr^c}+\la(\1_\Aa-\1_{\Aa^c})
\eeq
where $\{c_i\}_{i=1}^t,c$ are variables. In this case, we'll have $f(E^L,W_0)=\sum_{i=1}^t(c_i+\frac{1}{k_i})\1_{\Rr_{i,i}}+c\1_{\Rr^c}$ and when $c_i\leq -\frac{1}{k_i}$ and $c\geq 0$ using (\ref{ES}) we'll have $f(E^L,W_0)\geq 0$ for all $E^L$ as desired. $W_0$ is a random matrix where randomness is due to $\Aa$ and in order to ensure a small spectral norm we set its expectation to $0$.
Expectation of an entry of $W_0$ on $\Rr_{i,i}$ and $\Rr^c$ is $c_i+\la(2p_i-1)$ and $c+\la(2q-1)$ respectively. Hence
\beq
c_i=-\la(2p_i-1)~~~\text{and}~~~c=-\la(2q-1)
\eeq
and $f$ and $W_0$ take the following forms
\begin{align}
\label{fform2}
&f(E^L,W_0)= \la[(1-2q)\su(E_{\Rr^c})-\sum_{i=1}^t(2p_i-1-\frac{1}{\la k_i})\su(E_{\Rr_{i,i}})]\\
\label{w0form2}
&W_0=2\la [\sum_{i=1}^t(1-p_i)\1_{\Rr_{i,i}\cap\Aa}-p_i\1_{\Rr_{i,i}\cap\Aa^c}+(1-q)\1_{\Rr^c\cap\Aa}-q\1_{\Rr^c\cap\Aa^c}]
\end{align}
Hence we require $\la(2p_i-1)\geq \frac{1}{k_{min}}$ and $1\geq 2q$. Notice that $W_0$ has the same form (\ref{w0form}) analyzed previously. Consequently, Lemma \ref{sondanbir} directly applies and $\|W_0\|$ is bounded above by $2(1+\eps+o(1))\la\sqrt{n}$ w.h.p.

\subsubsection{Summary of section \ref{seccase4}}
Luckily, Lemma \ref{correction} also directly applies as form of the $W_0$ is exactly same as in section \ref{intelproof}. As a result, we can state the following Theorem.
\begin{thm}
\label{part4}
$W_0$ is as described previously in (\ref{w0form2}). Choose $W$ to be projection of $W_0$ on $\MM_\U$. Also set $\la=\frac{1}{4\sqrt{n}}$ and let $e$ be same as in Theorem \ref{part3} and assume $\{k_i\}$ is such that $e>0$.

Then, w.p.a.l. $1-6n^2\exp(-\frac{2}{9}e^2k_{min})-4\exp(-\frac{n}{100})$ we have
\begin{itemize}
\item $\|W\|<1$
\item For all feasible $E^L$, $f(E^L,W)\geq 0$. 
\end{itemize}
\end{thm}
\begin{proof}
Exactly similar to the proof of Theorem \ref{part2} w.p.a.l. $1-4\exp(-\frac{n}{100})$ we have $\|W\|<1$. Secondly from Lemma \ref{correction} w.p.a.l. $1-6n^2\exp(-\frac{2}{9}e^2k_{min})$ we have $\|W_0-W\|_\infty\leq 2\la e$. Then based on (\ref{fform2}) for all $E^L$
\begin{align}
f(E^L,W)&=f(E^L,W_0)-\li W_0-W,E^L\ri\geq f(E^L,W_0)-\la e(\su(E^L_\Rr)-\su(E^L_{\Rr^c}))\\
&=\la[(1-2q-e)\su(E^L_{\Rr^c})-\sum_{i=1}^t(2p_i-1-\frac{1}{\la k_i}-e)\su(E^L_{\Rr_{i,i}})]\geq 0
\end{align}
Hence by a union bound $W$ satisfies both of the desired conditions.
\end{proof}

\subsection{Final Step}
\begin{proof}[{\bf{Proof of Theorem \ref{mainthmb}}}]
Notice that $\la=\frac{1}{4\sqrt{n}}$ and $k_{i}\geq \frac{8\sqrt{n}}{2p_i-1}$ implies
\beq
2e=\min\{1-2q,\{2p_i-1-\frac{1}{\la k_{i}}\}_{i=1}^t\}\geq \min\{1-2q,\{\frac{2p_i-1}{2}\}_{i=1}^t\}=\min\{1-2q,p_{min}-1/2\}
\eeq
Then based on Theorems \ref{part3} and \ref{part4} w.p.a.l. $1-cn^2\exp(-C\left(\min\{1-2q,2p_{min}-1\}\right)^2k_{min})$
\begin{itemize}
\item For all nonzero $E^L\in\MM_\U^\perp$ we have $g(E^L)>0$.
\item There exists $W\in\MM_\U$ with $\|W\|<1$ s.t. for all $E^L$, $f(E^L,W)\geq 0$.
\end{itemize}
Consequently based on Lemma \ref{twocase}, $(L^0,S^0)$ is the unique optimal of problem \ref{optim3}.
\end{proof}

\section{Proof of Theorem \ref{conversethm}}
\label{secconv}
\begin{proof}[{\bf{Proof of Theorem \ref{conversethm}}}] For the proof, we'll construct a feasible $(L^1,S^1)$ which yields a lower objective value w.h.p. Consider the first case where $\frac{1}{2}\geq p_{min}$. WLOG assume $\{p_i\}$ is ordered decreasingly and $p_c>\frac{1}{2}\geq p_{c+1}$ for some $c\leq t-1$. Then, let $L^1=\sum_{i=1}^c\1_{\R_{i,i}}$ and $S^1=\A-L^1$. Then difference between objectives is given by
\beq
\label{coolineq}
\|L^0\|_\s-\|L^1\|_\s+\la(\|S^0\|_1-\|S^1\|_1)=\sum_{i=c+1}^t k_i+\frac{C}{\sqrt{n}}\su(\1_{\Aa^c\cap\Gamma}-\1_{\Aa\cap\Gamma})
\eeq
where $\Gamma=\bigcup_{i> c} \Rr_{i,i}$. $\su(\1_{\Aa^c\cap\Gamma}-\1_{\Aa\cap\Gamma})$ is simply summation of $|\Gamma|$ independent $\text{Bern}(1,-1,p_i)$ random variables (for some $i>c$). Hence, means are nonpositive as $p_i\leq \frac{1}{2}$ and we'll argue w.h.p. for all $i>c$ and $k_i\neq 0$
\beq
\label{coolineq2}
h(\C_i,\Aa)=\frac{k_i\sqrt{n}}{C} +\su(\1_{\Aa^c\cap\Rr_{i,i}}-\1_{\Aa\cap\Rr_{i,i}})>0
 \eeq
 to conclude. There are $k_i^2$ such random variables in $\Rr_{i,i}$ hence a Chernoff bound will give $\Prob[h(\C_i,\Aa)>0]\geq 1-c_1\exp(-c_2 n) $ for appropriate constants $c_1,c_2>0$ for any $k_i\neq 0$. The reason is, we need a deviation of at least $\frac{\sqrt{k_in}}{C}$ from the mean. By using union bound over events (\ref{coolineq2}), we obtain (\ref{coolineq}) is positive w.h.p.
 
 If $q>\frac{1}{2}$, let $L^1=\1$ and $S^1=-\1_{\Aa^c}$. Then
 \beq
\label{coolineq3}
\|L^0\|_\s-\|L^1\|_\s+\la(\|S^0\|_1-\|S^1\|_1)=\sum_{i=1}^t k_i-n+\frac{C}{\sqrt{n}}\su(\1_{\Aa\cap\Rr^c}-\1_{\Aa^c\cap\Rr^c})
\eeq
Note that $n-\sum_{i=1}^t=|\C_{t+1}|$ where $\C_{t+1}$ was the set of nodes outside of the clusters. 
Then, we just need to show that $\su(\1_{\Aa\cap\Rr^c}-\1_{\Aa^c\cap\Rr^c})>\frac{1}{C}|\C_{t+1}|\sqrt{n}$ to conclude that $(L^1,S^1)$ is strictly better. Similar to the previous case, $\su(\1_{\Aa\cap\Rr^c}-\1_{\Aa^c\cap\Rr^c})$ is sum of $|\Rr^c|$ $\text{Bern}(1,-1,q)$ random variables.

If $\C_{t+1}\neq\emptyset$: Clearly $|\Rr^c|\geq |\C_{t+1}|n$. Consequently, 
\beq
\E[\su(\1_{\Aa\cap\Rr^c}-\1_{\Aa^c\cap\Rr^c})]\geq |\Rr^c|(2q-1)\geq |\C_{t+1}|n(2q-1)
\eeq
and due to Chernoff bounding, it is highly concentrated around the mean. As $n\rightarrow \infty$ we have $|\C_{t+1}|n(2q-1)>>\frac{1}{C}|\C_{t+1}|\sqrt{n}$ hence, w.h.p. (\ref{coolineq3}) is positive. Error exponent is $\Omega(|\C_{t+1}|n)$.

On the other hand, if $\C_{t+1}=\emptyset$ but $|\Rr^c|\neq 0$ then $t\geq 2$ and we have $|\Rr^c|\geq 2(n-1)$ as for any nonzero integers $a,b$ with $a+b=n$
\beq
(a+b)^2-a^2-b^2=2ab\geq 2(a+b-1)=2(n-1)
\eeq 
In this case, we only require $\su(\1_{\Aa\cap\Rr^c}-\1_{\Aa^c\cap\Rr^c})>0$. Again, this will happen w.h.p. since $2q-1>0$. Error exponent is $|\Rr^c|$ which is $\Omega(n)$.
\end{proof}
%%Conclusions
%
%
%
%
%
%

\end{document}